\documentclass[sigconf,nonacm]{acmart}
\usepackage{times} 

\usepackage{amsmath}
\usepackage{amsfonts}
\usepackage{graphicx}
\usepackage{color}
\usepackage{url}
\usepackage[algoruled,linesnumbered]{algorithm2e}
\usepackage{paralist} 
\usepackage{listings} 
\usepackage{bbding} 
\usepackage[draft,commandnameprefix=always]{changes}
\definechangesauthor[name=Christos, color=orange]{CF}

\usepackage{subcaption}
\usepackage[table]{xcolor}

\SetCommentSty{mycommfont}

\usepackage{diagbox}

\setcopyright{acmlicensed}
\copyrightyear{2025}
\acmYear{2025}
\acmDOI{XXXXXXX.XXXXXXX}
\acmConference[]{}{September 18, 2025}{}

\newtheorem{problem}{Problem}

\newtheorem{lemma}{Lemma}

\newcommand{\hide}[1]{}

    \newgeometry{margin=1in}

\newcommand{\bit}{\begin{compactitem}}
\newcommand{\eit}{\end{compactitem}}
\newcommand{\ben}{\begin{compactenum}}
\newcommand{\een}{\end{compactenum}}

\newcommand{\method}{\textsc{FraudGuess}\xspace}
\newcommand{\methodone}{\textsc{\method-D}\xspace}
\newcommand{\methodtwo}{\textsc{\method-J}\xspace}

\newcommand{\insights}{{Insights\xspace}}
\newcommand{\discoveries}{{Discoveries\xspace}}

\newcommand{\afi}{{AFI}\xspace}
\newcommand{\afifull}{{Anonymous Financial Institution (\afi)\xspace}}

\newcommand{\afidataset}{{\texttt{\afi-txns}\xspace}}

\newcommand{\dmg}{{``Double Machine-gun''\xspace}}
\newcommand{\emphdmg}{{\em \dmg}}
\newcommand{\dmgexample}{{\em `93522'}}
\newcommand{\ph}{{``Penny Hunter''\xspace}}
\newcommand{\emphph}{{\em \ph}}
\newcommand{\phexample}{{\em `6269035'}}
\newcommand{\cp}{{``Bursty Poster''\xspace}}
\newcommand{\emphcp}{{\em \cp}}
\newcommand{\cpexample}{{\em `5653936'}}

\newcommand{\dataset}{\mathbf{D}}
\newcommand{\numfts}{k}
\newcommand{\ft}[1]{f_{#1}}
\newcommand{\fti}{\ft{i}}
\newcommand{\ftj}{\ft{j}}

\newcommand{\txn}[1]{t_{#1}}
\newcommand{\txni}{\txn{i}}
\newcommand{\txnj}{\txn{j}}
\newcommand{\txnk}{\txn{k}}

\newcommand{\detection}{{Detection\xspace}}
\newcommand{\emphdetection}{`{\em \detection}'}
\newcommand{\justification}{{Justification\xspace}}
\newcommand{\emphjustification}{`{\em \justification}'}

\newcommand{\gone}{{New Fraud Type Detection\xspace}}
\newcommand{\gtwo}{{Explainable\xspace}}
\renewcommand{\gone}{\emphdetection}
\renewcommand{\gtwo}{\emphjustification}

\newcommand{\scale}{{Scalable\xspace}}
\newcommand{\effective}{{Effective\xspace}}
\newcommand{\automatic}{{Parameter-Free\xspace}}
\newcommand{\financial}{{Feat. Sel. for Financial Setting\xspace}}
\renewcommand{\financial}{{Cust. for Fin. data\xspace}}

\newcommand{\emphasize}[1]{\textbf{\underline{\smash{#1}}}}

\newcommand{\gold}[1]{\cellcolor{green}{#1}}
\newcommand{\silver}[1]{\cellcolor{green!25}{#1}}

\newcommand{\targetCard}{{\em `target'}\xspace}

\begin{document}

\title{ \method: Spotting and Explaining New Types of Fraud 
in Million-Scale Financial Data}

\author{Robson L. F. Cordeiro}
\affiliation{%
  \institution{
  Carnegie Mellon University}
  \city{Pittsburgh, PA}
  \country{USA}
}
\email{robsonc@andrew.cmu.edu}

\author{Meng-Chieh Lee}
\affiliation{%
  \institution{
  Carnegie Mellon University}
  \city{Pittsburgh, PA}
  \country{USA}
}
\email{mengchil@cs.cmu.edu}


\author{Christos Faloutsos}
\affiliation{%
  \institution{
  Carnegie Mellon University}
  \city{Pittsburgh, PA}
  \country{USA}
}
\email{christos@cs.cmu.edu}

\begin{abstract}


Given a set of financial transactions (who buys from whom, when, and for how much),
as well as prior information from buyers and sellers,
how can we find fraudulent transactions?
If we have labels for some transactions for known types of fraud, 
we can build a classifier. 
However, we also want to find {\bf new} types of fraud,
still unknown to the domain experts (\emphdetection).
Moreover, we also want to provide evidence to experts
that supports our opinion (\emphjustification).
In this paper, we propose \method, to achieve two goals:
(a) for \emphdetection, it spots new types of fraud as micro-clusters
in a carefully designed feature space;
(b) for \emphjustification, it uses visualization and heatmaps for evidence, 
as well as an interactive dashboard for deep dives.

\method is used in real life and is currently considered for deployment
in an {\em \afifull}.
Thus, we also present the {\bf three} new behaviors that \method
discovered in a real, {\bf million}-scale financial dataset.
Two of these behaviors are deemed fraudulent or suspicious
by domain experts, catching {\bf hundreds}
of fraudulent transactions
that would otherwise go un-noticed.

\end{abstract}

\maketitle


\section{Introduction}
\label{sec:intro}


How can we spot {\em new types} of fraud
in financial transaction data?
How can we also provide the domain 
experts with intuitive evidence, 
which could convince a non-technical jury/judge?


\begin{figure*}[t!]
    \centering
    \begin{subfigure}[t]{0.3\textwidth}
        \centering
        \includegraphics[height=2.45in]{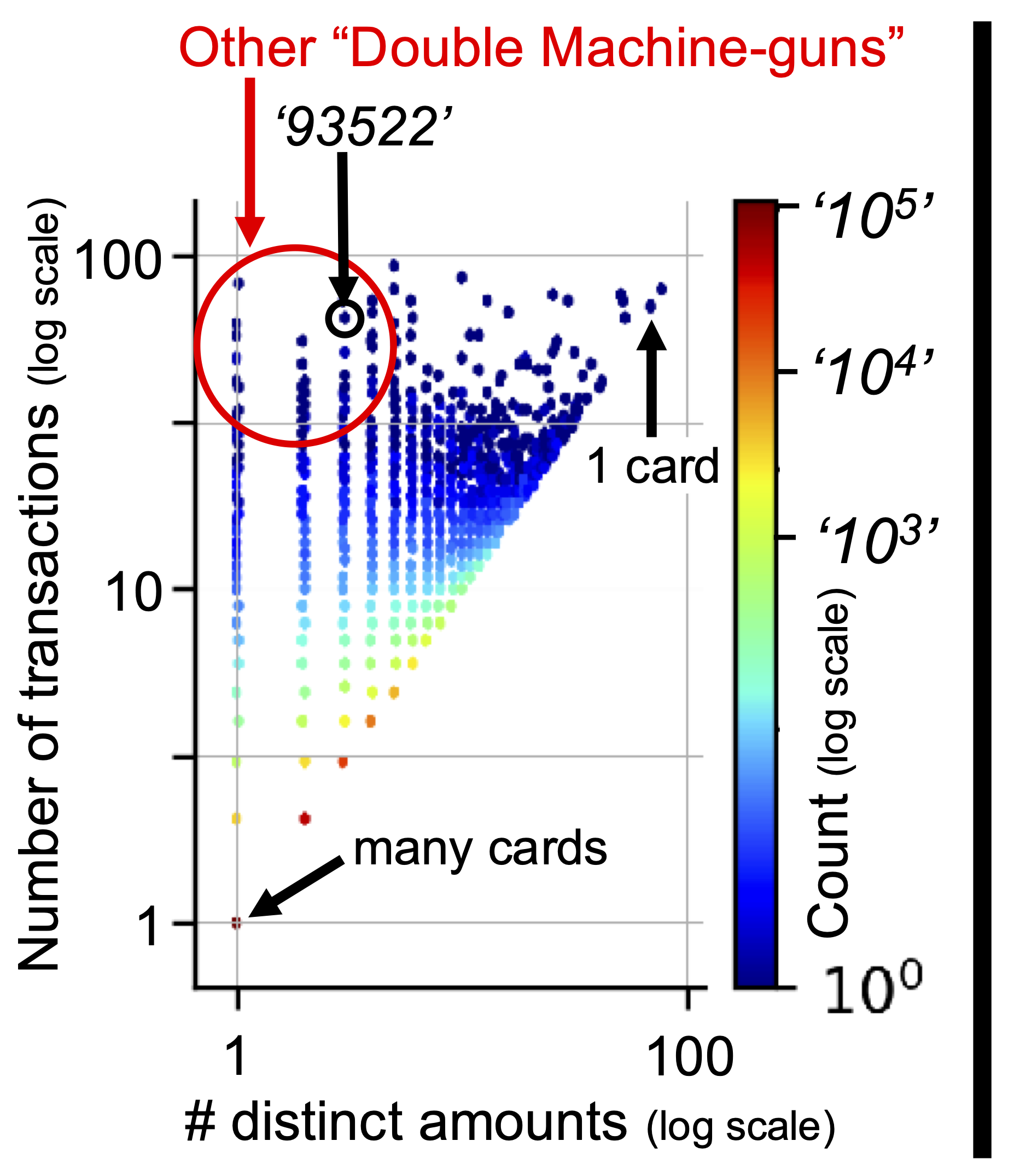}
        \caption{Goal G1 -- \emphdetection}
        \label{fig:dmg_others}
    \end{subfigure}%
    ~ 
    \begin{subfigure}[t]{0.7\textwidth}
        \centering
        \includegraphics[height=2.45in]{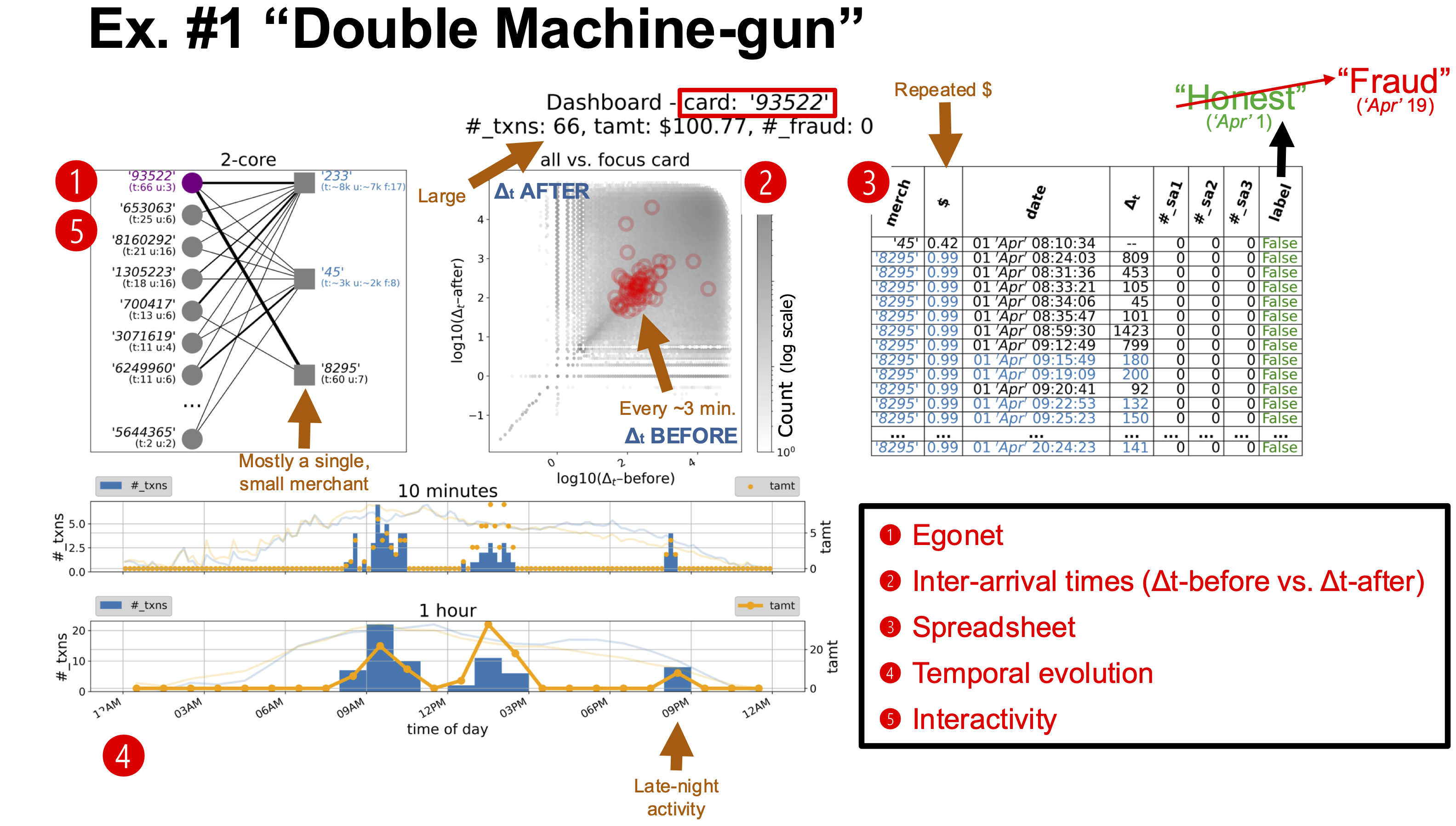}
        \caption{Goal G2 -- \emphjustification}
        \label{fig:dmg_example}
    \end{subfigure}
    \caption{\emphasize{\method found a NEW, suspicious behavior and justified its decisions.}
    (a) \emphdetection~(Goal G1): Our method caught tens of suspicious cards~(points inside the red circle) exhibiting the new behavior \dmg, with many, synchronized txns of the same amount.
    (b) \emphjustification~(Goal G2): \method justified its decisions via visual inspection. We showcase the interactive dashboard of the suspicious (and later, confirmed fraudster) card \dmgexample~detected before; note an unusual behavior with $66$ txns/day, often of $\$0.99$, at every $\sim\hspace{-0.5mm}3$ min. with a small merchant, and occasionally late at night.
    \label{fig:crownJewel}}
\end{figure*}

The research problems we want to solve are as follows:
\begin{problem}[\textbf{Generalized Fraud Detection}]\label{def:problem}
Given:
\begin{compactitem}
    \item Financial transactions (or simply `txns', for short) with card ID, merchant ID, amount, and timestamp,
    \item Features for the cards and merchants (optional), and
    \item Fraud/non-fraud labels for a few txns (optional).
\end{compactitem}
Address the following goals:
\begin{compactitem}
    \item (G0 - Classification: Find fraud txns similar to those already labeled),
    \item G1 - \emphdetection: Find \textbf{new} types of fraud, and
    \item G2 - \emphjustification: Explain our decisions.
\end{compactitem}
\end{problem}



\par \noindent
Since the \afi is already doing classification successfully (Goal G0),
we do not emphasize classification in this work.
For the remaining two goals, 
our main contributions are:
\bit
\item {\bf \emphdetection} (Goal G1): \methodone
finds {\em new types} of fraudulent behavior, before domain experts
are even aware of it.
\item {\bf \emphjustification} (Goal G2): \methodtwo 
uses visualization to explain and justify our decisions,
as well as an interactive dashboard for domain experts
to do deep dives and discover more evidence.
\eit
Our analysis by \method, revealed the following:
\bit
    \item {\bf \insights}: we discovered that $3$ to $5$ features (out of $400+$ available) are generally enough for fraud detection, including the $\#$ of txns, median amount, $\#$ of distinct amounts, and $\#$ of distinct inter-arrival times;
\item {\bf \discoveries}: we found three {\em new}, unusual behaviors \dmg, \ph, and \cp, with the first two being confirmed later as fraud or 
highly suspicious 
by our \afi~collaborators.
\eit

\paragraph{Dataset Anonymization:}
Before we illustrate some of the results of our method in Figure~\ref{fig:crownJewel}, 
we need to describe our anonymization steps, since
 the dataset contains sensitive financial data. 
In fact, all the card-IDs, all the merchant-IDs, and all the years and months
of the timestamps are hashed or obfuscated, indicated by
{\em `italics'} font and quotes.
All the other data (amount, day-of-month, time-of-day, etc.) are actual data.


\paragraph{Example of our results:}
Figure~\ref{fig:crownJewel} showcases a few results from \method.
For \emphdetection~(Goal G1), we carefully selected features and capitalize on heatmaps
to catch a \textit{new}, suspicious behavior: \dmg, that is, cards with many synchronized txns of the same amount.
Figure~\ref{fig:dmg_others} shows tens of real examples we caught (points inside the red circle).
For \emphjustification~(Goal G2), we leverage an interactive dashboard that deep dives on any card by showing its egonet, inter-arrival times and time of day of txns, etc., to explain and justify our decisions through visual inspection.
For example, Figure~\ref{fig:dmg_example} deep dives on the suspicious card \dmgexample~detected before.
Thanks to \method, we note unusual behavior with $66$ txns, mostly of $\$0.99$ with a single, small merchant, occurring at every $\sim3$ minutes within a day, and occasionally late at night.
Initially considered honest, these txns were later confirmed as fraud by the \afi~ collaborators, after they received an alarm from \method~and reviewed the case with its dashboard.
All the details are given later, in Section~\ref{sec:exp}.

{
{\bf Reproducibility}: Pending approval by the \afifull, 
we will open-source our code along with synthetic data in the appropriate format.
}

The outline of the paper is typical: 
we give the survey (Section~\ref{sec:background}),
the proposed method (Section~\ref{sec:meth}),
experiments (Section~\ref{sec:exp}),
and conclusions (Section~\ref{sec:concl}).


\section{Background and Related Work}
\label{sec:background}


The related work forms the following groups:
clustering, classification, and
anomaly detection methods.
However, 
\method is the {\em only one} that satisfies all properties, 
as shown in Table~\ref{tab:salesman}.
Notice that we explicitly stay away from `black-box' methods (e.g., deep learning and transformers) for explainability.

\paragraph{Clustering:}
Unsupervised methods, such as K-Means \cite{lloyd1982least}, 
DBSCAN \cite{ester1996density}, and OPTICS \cite{ankerst1999optics}, 
do not need class labels; they group similar data elements into clusters,
thus summarizing the dataset for the analyst.
As a side product, they also provide a list of outliers,
that don't fit in any of the clusters.

\paragraph{Classification:}
These algorithms handle the case of supervised learning:
they expect a training set with class labels,
and build a model to categorize the rest of the data entries.
Successful representatives include
AutoGluon \cite{agtabular} and 
XGBoost \cite{chen2016xgboost}.
We will not consider these methods, as they are unable
to spot {\em new} types of fraud (Goal G1 -- \emphdetection), which is our main focus.

\paragraph{Anomaly Detection:}
Traditional anomaly detection methods, such as Isolation Forest \cite{liu2008isolation} and RRCF \cite{guha2016robust}, focus on spotting point-anomalies.
Recently, an emerging problem is to spot group-anomalies, also known as micro-clusters.
Notable methods addressing this issue include \textsc{Gen$^2$Out} \cite{lee2021gen}, D.MCA \cite{jiang2022d}, and \textsc{McCatch} \cite{vinces2024mccatch}.
These approaches are more practical in financial fraud detection settings, where fraudulent transactions from the same micro-cluster are likely to be conducted by organized crime.

In addition, anomaly detection in graphs has been extensively studied \cite{akoglu2015graph}, focusing on identifying dense subgraphs \cite{DBLP:conf/kdd/HooiSBSSF16, DBLP:conf/icdm/ShinEF16}, frequent subgraphs \cite{lee2021gawd, lee2024descriptive}, and specific patterns such as ``smurfing'' \cite{lee2020autoaudit}.
Other relevant works, such as \textsc{CallMine} \cite{cazzolato2023callmine}, TgraphSpot \cite{cazzolato2022tgraphspot}, and TGRAPP \cite{cazzolato2023tgrapp}, detect and explain fraud in temporal graphs through visualizations; however, they are designed for phone call graphs and do not apply to financial graphs.


\begin{table}[htbp]
\caption{\emphasize{\method matches all specs}, while the competitors
miss one or more of the features. `?' means `depends',
or `not always possible'. \label{tab:salesman}}
\centering{\resizebox{1\columnwidth}{!}{
\begin{tabular}{ l| c | c | c | c || c}
   \diagbox{Property}{Method} & \rotatebox{80}{Clustering \cite{lloyd1982least, ester1996density, ankerst1999optics}} & \rotatebox{80}{Classification \cite{agtabular, chen2016xgboost}} & \rotatebox{80}{Anomaly Detection \cite{liu2008isolation, guha2016robust}} & \rotatebox{80}{\textsc{CallMine} \cite{cazzolato2023callmine}} & \rotatebox{80}{\Large \textbf{\method}} \\ 
\hline  
    G1: \emphdetection & \textbf{?} &  & \textbf{?} & \CheckmarkBold & \CheckmarkBold \\
    G2: \emphjustification &  &  &  & \CheckmarkBold & \CheckmarkBold \\ 
	\scale & \textbf{?} & \textbf{?} & \textbf{?} & \CheckmarkBold & \CheckmarkBold \\ 
    \automatic & \textbf{?} & \textbf{?} & \textbf{?} & \CheckmarkBold & \CheckmarkBold \\ 
    \financial &  &  &  &  & \CheckmarkBold \\
\end{tabular} 
}}
\end{table}

\section{Proposed Method}
\label{sec:meth}



In this section, we propose \method.
By analyzing the financial dataset \afidataset\xspace from \afi with \method, we reveal several surprising insights and discoveries.
As a reminder, we want to find {\em new types} of fraud (Goal G1: \emphdetection)
and to provide justification (Goal G2: \emphjustification).

The insights behind our \method are as follows:
\begin{compactitem}
    \item For \emphdetection, our main insight is to look 
    for lockstep behavior: fraudsters will often take the 
    same actions multiple times, often in rapid succession,
    and/or through multiple `sock-puppet' accounts.
    Thus, our insight is to look for micro-clusters
    in appropriate feature spaces.
    \item For \emphjustification, our main insight is to use 
    visualization and interaction.
    Thus, we stay away from black-box models 
    (like deep learning and transformers),
    and instead we provide several scatter plots and heatmaps,
    that need no statistical background to be understood
    (see Figure~\ref{fig:crownJewel}).
    \item A subtle issue is which features to use.
    We elaborate on this next.
\end{compactitem}

\hide{
\begin{table}[htbp]
\begin{center}
\begin{tabular}{|c|c|}
\hline  
Symbols & Definitions \\ 
\hline
$G$  & a graph \\ 
\hline  
$A$  &  adjacency matrix\\ 
\hline 
\end{tabular} 
\caption{Symbols and Definitions \label{tab:dfn}}
\end{center}
\end{table}
}

\subsection{Feature Selection}
\hide{
We have some labels.
Instead of using classification, we use forward feature selection with Isolation Forest to select the best features for anomaly detection.
By only using 3 out of 400+ features, our method remains competitive performance, while being capable of detecting new types of frauds.
...
} 
The \afi is already using $O(100)$ features for each transaction,
such as the dollar amount, 
whether the cardholder has been blocked in the past few days,
etc.
Such a number of features leads to the {\em curse of dimensionality} --
should we use them all? Should we drop some of them?
Should we add some more features?\looseness=-1

Our answer is two-fold: 
\begin{compactitem}
\item (a) we should do feature selection, and 
\item (b) we should add features that help us spot lockstep behavior.
\end{compactitem}

For the former, we applied forward selection using the class labels;
 we were surprised to see that out of the hundreds of features,
 only a handful were good enough (see Appendix~\ref{app:features}).
 
For the latter, the main insight is that fraudsters often do repetitive behaviors
(charge the same amount, every 10 seconds; or every day at 8am);
or they have multiple sock-puppet accounts, each charging the same small amount
to the same few colluding merchants, trying to keep everything `below the radar'.
In short, we try to find new types of fraud,
by spotting strange, repetitive (`lockstep') behavior.

Based on the above insights,
we propose the following new, carefully selected features:
\begin{compactitem}
\item Inter-arrival time (to catch repetitive 
      `machine-gun-like' behaviors).
\item Number of distinct amounts (again, to catch fraudsters that
    charge exactly the same small amount multiple times, as opposed
    to one large amount, to evade detection).
\item Time-of-day (for fraudsters that always operate 
    outside business hours, again to evade detection).
\end{compactitem}
We use functions of the above quantities (e.g. median/variance
of inter-arrival time, median/variance of amounts).

\subsection{\method Overview}

Algorithm~\ref{algo:main} shows the pseudo code of our \method.
It has two parts, one for each of the design goals:

\paragraph{Goal G1 \gone - \methodone} For the first part,
\method receives a set of txns $\dataset$ with card ID, merchant ID, amount, and timestamp; and then, it returns the most suspicious cards in~$\dataset$, with labels assigned according to distinct types of unusual behavior.
We first extract features $\ft{1},\dots \ft{\numfts}$ for each card in $\dataset$ (see Line $1$).
As shown later in Section~\ref{sec:g1},
four features ($\numfts=4$) are considered in our current implementation, although other features can be added at any time based on the domain knowledge of the users.

\paragraph{Goal G2 \gtwo - \methodtwo}
For the second part, 
we propose visualization with heatmaps to detect the most suspicious cards (see Lines~$2$ to $5$).
For each pair of features $\fti, \ftj \in \{\ft{1},\dots \ft{\numfts}\}:~i < j$,
we plot $\fti$ versus $\ftj$ in log-log scale using colors to distinguish variations in the density/overplotting of points.
Then, we ask the user to mark outlying points (cards) in the plot, e.g., by drawing an ellipse over them, if there are any.
Figure~\ref{fig:dmg_others} depicts an example heatmap created from dataset \afidataset, where a few outliers are marked by an ellipse.
Then, in Line~$6$, the user assigns labels to the groups of cards marked previously (group labeling).
In Lines~$7$ to $9$, an optional step displays dashboards of a few select cards of interest (e.g., one card per label type) for explanation and justification of the decisions through visual inspection.
Figure~\ref{fig:dmg_example} depicts an example dashboard created for a card from dataset \afidataset.
Finally, we return the cards marked and their labels in Line~$10$.

We provide the details of our \method in the following.

\begin{algorithm}[htbp]
\SetAlgoLined
\LinesNumbered
\KwData{A set of txns $\dataset$ with card ID, merchant ID, amount, and timestamp}
\KwResult{The most suspicious cards in $\dataset$, with labels}
\tcp{\methodone for Goal G1: \gone}
Extract feats. $\ft{1},\dots \ft{\numfts}$ for each card in $\dataset$; \tcp{$\numfts=4$}
\For{$\fti, \ftj \in \{\ft{1},\dots \ft{\numfts}\}:~i < j$}{
    Display a heatmap of $\fti$ versus $\ftj$ in log-log scale\;
    User marks outlying points (cards), if any\;
}
User assigns labels to marked cards (group labeling)\;
\tcp{\methodtwo for Goal G2: \gtwo}
\If{required by the user}{
Show dashboards of a few cards (e.g.,$1$ per label~type)\;
}
\Return marked cards and labels;
\caption{\method \label{algo:main}}
\end{algorithm}

\subsection{\methodone for G1 - \detection}
\label{sec:g1}
The main idea behind our \methodone is to detect automated (`lockstep') behavior:
if we see a card with many txns exhibiting signs of repetition/automation 
(e.g., same inter-arrival time, fixed amount, or always small amounts),
then we suspect fraud.
We carefully designed features to quantify these signs, and then we use visualization with heatmaps to detect the few most suspicious cases among millions of cards/txns in a fast and scalable manner.
Our features are carefully designed, 
so that they can be quickly (linearly) extracted from the set of txns of a card,
and at the same time provide information that can spot new types of fraud.

The proposed features are the following:
\begin{compactitem}
    \item \textbf{$\#$ of txns}: cardinality of the set of txns.
    \item \textbf{$\#$ of distinct inter-arrival times}: count of distinct inter-arrival times between pairs of consecutive txns, where the inter-arrival time is the time spent between two txns.
    \item \textbf{$\#$ of distinct amounts}: count of distinct amounts of txns.
    \item \textbf{Median amount}: median of the amounts of txns.
\end{compactitem}

\begin{figure*}[t!]
    \centering
    \begin{subfigure}[t]{0.7\textwidth}
        \centering
        \includegraphics[height=1.7in]{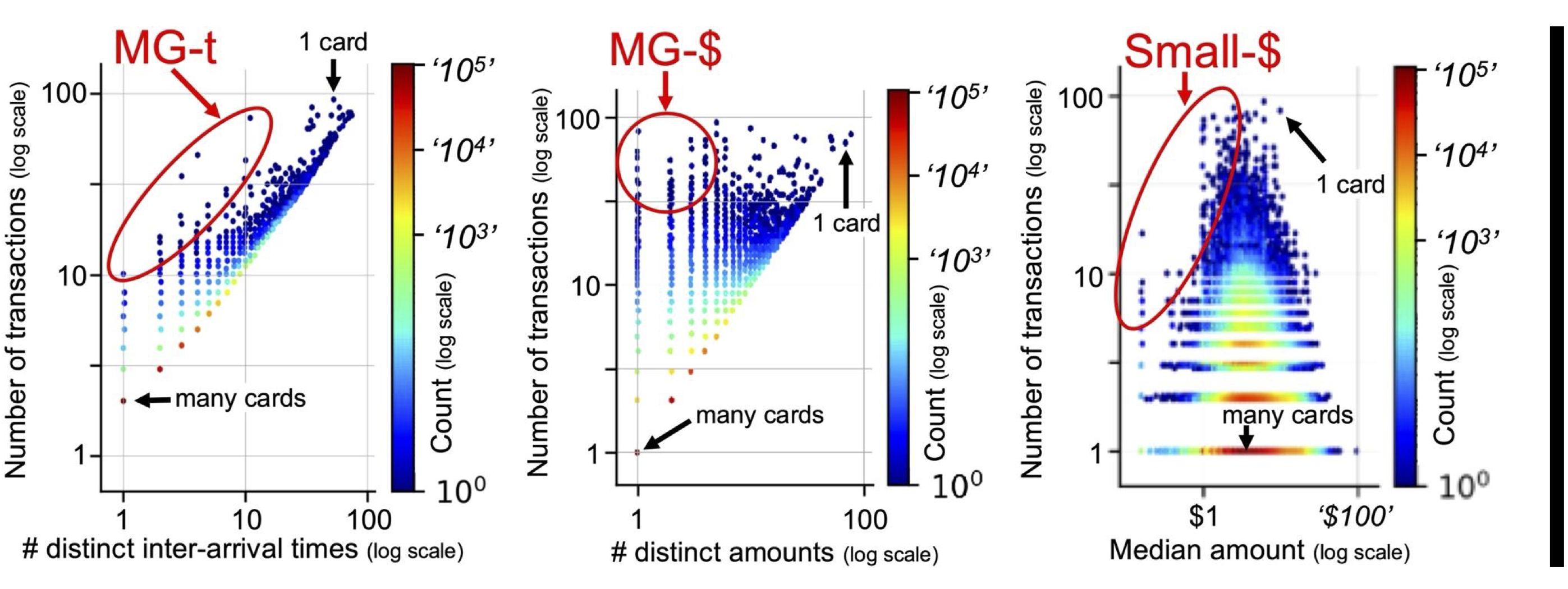}
        \caption{Heatmaps for detection.
        \label{fig:g1_plot}}
    \end{subfigure}%
    ~ 
    \begin{subfigure}[t]{0.3\textwidth}
        \centering
        \includegraphics[height=1.7in]{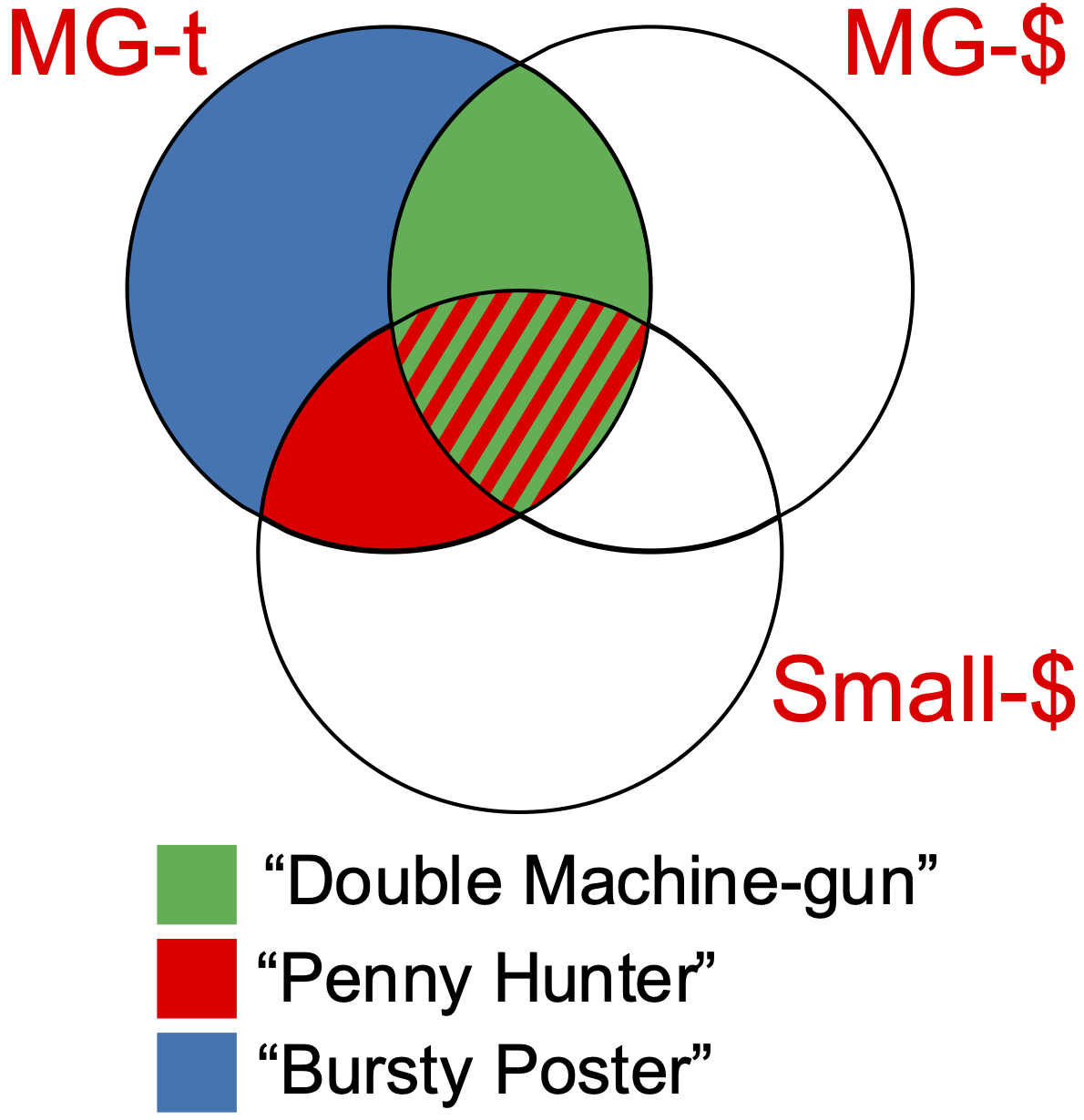}
        \caption{Venn diagram of strange behaviors.}
        \label{fig:g1_veen}
    \end{subfigure}
    \caption{ \label{fig:g1}
    Goal G1 -- \detection: 
    'MG-t': machine-gun behavior over time (e.g., every few seconds)
    'MG-\$': ditto, over amounts;
    'small-\$': unusually small, and repetitive, amounts.}
\end{figure*}

We use the count of distinct values to quantify repetition of inter-arrival times and amounts because it is much simpler than other options, like the entropy, and still very effective, as we show later in the paper.
We also employ the median amount instead of other types of average because it is robust to outliers.

\paragraph{More details:}
After extracting features for the cards, we use visualization with heatmaps to detect suspicious cases.
The reason we use heatmaps is over-plotting: several cards end up having the exact same features.
Notice that the color-scale of the heatmaps is {\em logarithmic}: `blue' indicates a single card,
while `red' indicates hundreds (or more) of cards.
Both axes of all heatmaps are also {\em logarithmic}, since we expected
(and found) multiple power laws and skewed distributions.

Also notice that we are currently doing {\em manual} detection of outliers and suspicious
points in our heatmaps.
In the near future, we plan to use some of the outlier detection methods in the background
Section~\ref{sec:background}, like Isolation Forest, \textsc{Gen$^2$Out}, etc.

Next, we describe the steps we used to spot the three new behaviors.
We first plot $\#$ of txns versus $\#$ of distinct inter-arrival times in log-log scale using colors to distinguish variations in density/overplotting of points.
Then, we ask the user to mark points in the top-left part of the plot (e.g., by drawing an ellipse over them, if there is any) because these points refer to cards with many txns and few distinct inter-arrival times.
We call this characteristic \textit{``Machine-gun in Time''} (or simply MG-t, for short) since it indicates cards with repetitive actions like machine guns, such as a card posting $100$ txns at every $5$ seconds ($\#$~of txns:~$100$; $\#$~distinct inter-arrival times:~$1$).
Figure~\ref{fig:g1_plot}~(left) depicts the plot for dataset \afidataset.
Note the several outlying, blue points marked with an ellipse over them in the plot's top-left. Each one is a card exhibiting signs of MG-t, which is unusual since the densest (redder) regions showing what is ``normal'' for the cards are in the bottom-left of the plot.

We do the same manual inspection to discover two additional signs of repetition/automation:
\begin{compactitem}
    \item \textit{``Machine-gun in Amount''} (MG-$\$$): we detect cards acting as machine guns with repetitive amounts by plotting $\#$~of txns versus $\#$ of distinct amounts in a heatmap. Note the several examples detected in \afidataset, and marked with an ellipse over them in Figure~\ref{fig:g1_plot}~(middle).
    \item \textit{``Small-amount Only''} (Small-$\$$): we uncover cards posting mostly txns of small amounts by plotting $\#$~of txns versus median amount in a heatmap. Note the tens of examples detected in \afidataset, and highlighted with an ellipse over them in Figure~\ref{fig:g1_plot}~(right).
\end{compactitem}

We finish the detection by following the Venn diagram in Figure~\ref{fig:g1_veen}.
It leads us to three unusual behaviors that were previously unknown to the \afi~human experts.
Specifically, we classify cards with signs of MG-t and MG-$\$$ as \emphdmg,
because they act as machine guns with regard to both the inter-arrival times and the amounts of the txns.
Class \emphph regards cards with signs of MG-t and Small-$\$$, as they mostly post small-amount txns in a synchronized manner.
Finally, cards exhibiting only MG-t are classified as \emphcp.
They show automated behavior by synchronously posting many txns, but may be seen as less suspicious than cards of the previous two classes because no other sign of automation is apparent.

\subsection{\methodtwo for G2 - \gtwo} 

\paragraph{Choice of information to present:}
We present a novel, interactive dashboard for domain experts to do deep dives on any suspicious card 
(referred to as the \targetCard card)
and discover more evidence through visualization.
Our dashboard has five main functionalities, all carefully designed to spot lockstep and suspicious behavior:
\begin{compactitem}
    \item \textit{Egonet}: to spot collusions.
    \item \textit{Inter-arrival times ($\Delta_t$-before versus $\Delta_t$-after)}: to spot machine-gun over time.
    \item \textit{Spreadsheet}: to give full information.
    \item \textit{Temporal evolution}: to spot unusual time-of-day activity.
    \item \textit{Interactivity}: to ``navigate'' between cards interactively.
\end{compactitem}
We provide the details of each functionality in the following.

\paragraph{Egonet:} Figure~\ref{fig:dmg_example} (top left).
Here we catch collusion, by spotting cards with coordinated activity, merchants acting together to achieve illicit gains, and even card (owners) colluded with merchants.
To this end, we analyze the two-step-away egonet of the \targetCard card, and display its main core (non-empty $k$-core with largest $k$) along with the corresponding $k$.
The {\em core number} or {degeneracy} of a node is defined in the footnote
\footnote{
The core number of
a node is the largest value $k$ of a $k$-core containing that node;
a $k$ core is a subgraph where all nodes  have at least $k$ neighbors in the subgraph.
Notice that computing the core-number of all nodes is a fast, linear operation.};
the intuition is that a high core number means that the node is well connected,
which usually means that the node is part of coordinated, lockstep activity.
Nodes are shaped as circles and squares to represent cards and merchants, respectively.
Each node is shown along with the corresponding ID, $\#$ of txns, $\#$ of unique counterparties, and $\#$ of confirmed fraudulent txns (if there are any and labels are provided); the last three values are given in parentheses, in that order.
The node of the \targetCard card is highlighted in purple, and all other nodes are shown in gray.
Edge widths are defined by the count of txns between each card and merchant, specifically using the log of the count for better visualization of cases with small and large counts.
Confirmed fraud is shown in separate red edges.\looseness=-1

Figure~\ref{fig:dmg_example}~(top left) shows the egonet of an example \dmg~card \dmgexample.
Thanks to \methodtwo, we note that the majority of the $66$ txns of this card were made with a single, small merchant, which may indicate collusion between them.\looseness=-1

\paragraph{Inter-arrival times ($\Delta_t$-before versus $\Delta_t$-after):} Figure~\ref{fig:dmg_example} (top middle).
Here we detect machine-gun over time by creating a plot to depict the inter-arrival times of txns.
We plot $\Delta_t$-before versus $\Delta_t$-after in log-log scale for every triplet of consecutive txns $<\txni,\txnj,\txnk>$ of the target card. $\Delta_t$-before is the inter-arrival time between $\txni$ and~$\txnj$. $\Delta_t$-after is the inter-arrival time between $\txnj$ and $\txnk$.
Semi-transparent red circles are used for the markers, thus leveraging transparency and overplotting to distinguish variations in the density of points.
Importantly, points close to the plot's main diagonal show repetition of inter-arrival times.
Hence, machine-gun behavior of the \targetCard card is evidenced through many red circles overplotting with each other along the main diagonal.
We also allow comparison with the ``normal'' behavior by adding an underneath, gray-scale heatmap built from the txns of every other card.

Figure~\ref{fig:dmg_example}~(top middle) depicts the inter-arrival times for our 
\targetCard card \dmgexample. 
Machine-gun over time is clearly evidenced through many red circles along the plot's main diagonal.
Particularly in this case, the txns occurred mostly at every $\sim3$ minutes.

\paragraph{Spreadsheet:} Figure~\ref{fig:dmg_example} (top right).
Here we give the full information for a few txns of the \targetCard card.
The txns are shown in chronological order with merchant ID, amount, timestamp, and inter-arrival time,
as well as
counts of suspicious activities (indicated by \afi) - 
for example, count of failed authentications during last week, last day, etc.
The last column is the fraud labels, if they exist.
We use colors to highlight repeated merchants, amounts, and inter-arrival times,
as well as non-zero counts of suspicious activities and cases of confirmed fraud.

Figure~\ref{fig:dmg_example}~(top right) displays the spreadsheet for our 
\targetCard card \dmgexample.
It shows that we have:
(a) machine-gun in the amount, with mostly $\$0.99$ txns;
(b) a preference to a single merchant ({\em `8295'}), and;
(c) no cases of fraud previously known to \afi.

\paragraph{Temporal evolution:} Figure~\ref{fig:dmg_example} (bottom left).
Here we catch unusual time-of-day activity.
We group the txns of the target card by their timestamps considering bins of length $l$ minutes.
Then, we create a two-y-axis plot to show both the total $\#$ of txns (leftmost y-axis, with blue bars) and the total amount (rightmost y-axis, with orange points) of each group over time (x-axis).
We also allow comparison with the ``normal'' behavior by adding two curves to the plot; they show the combined activity of all other cards over time, with total $\#$ of txns and total amount shown by a blue and an orange curve, respectively.
For better visualization, we normalize the combined activity according to that of the \targetCard card.
The whole process is repeated twice, with bins of a fine ($l=10$) and a coarse ($l=60$) granularity, thus creating two plots shown on top of each other.

Figure~\ref{fig:dmg_example}~(bottom-left) depicts the temporal evolution for our example card \dmgexample.
Most of the activity goes as expected, that is, txns in business hours with a lunch break.
Thanks to \methodtwo, we also catch an unusual activity with several txns around $8$~PM.

\paragraph{Interactivity:} Figure~\ref{fig:dmg_example} (top left).
Our dashboard is interactive. This means that the user can click on any card shown in the egonet of the \targetCard card, and then see the updated dashboard for that other card.
Hence, the user can ``navigate'' between cards of interest, interactively.



%
%
\vspace{5mm}
\subsection{Complexity Analysis}
\begin{lemma}
\methodone requires time linear on the input size.
\end{lemma}

\begin{proof}
    (Sketch) All the features we extract are carefully
    selected, so that they require linear time
    on the number of txns - 
    most of them requiring just a single pass 
    over the dataset.
\end{proof}

Figure ~\ref{fig:scale} illustrates
the linear behavior.
It plots the number of txns versus the run time in seconds.
Notice that our \method scales linearly on the input size, 
and takes less than 10 minutes to run for 
3.2M transactions on a stock CPU server.

\begin{figure}[hbtp]
\centering
\includegraphics[width=0.4\textwidth]{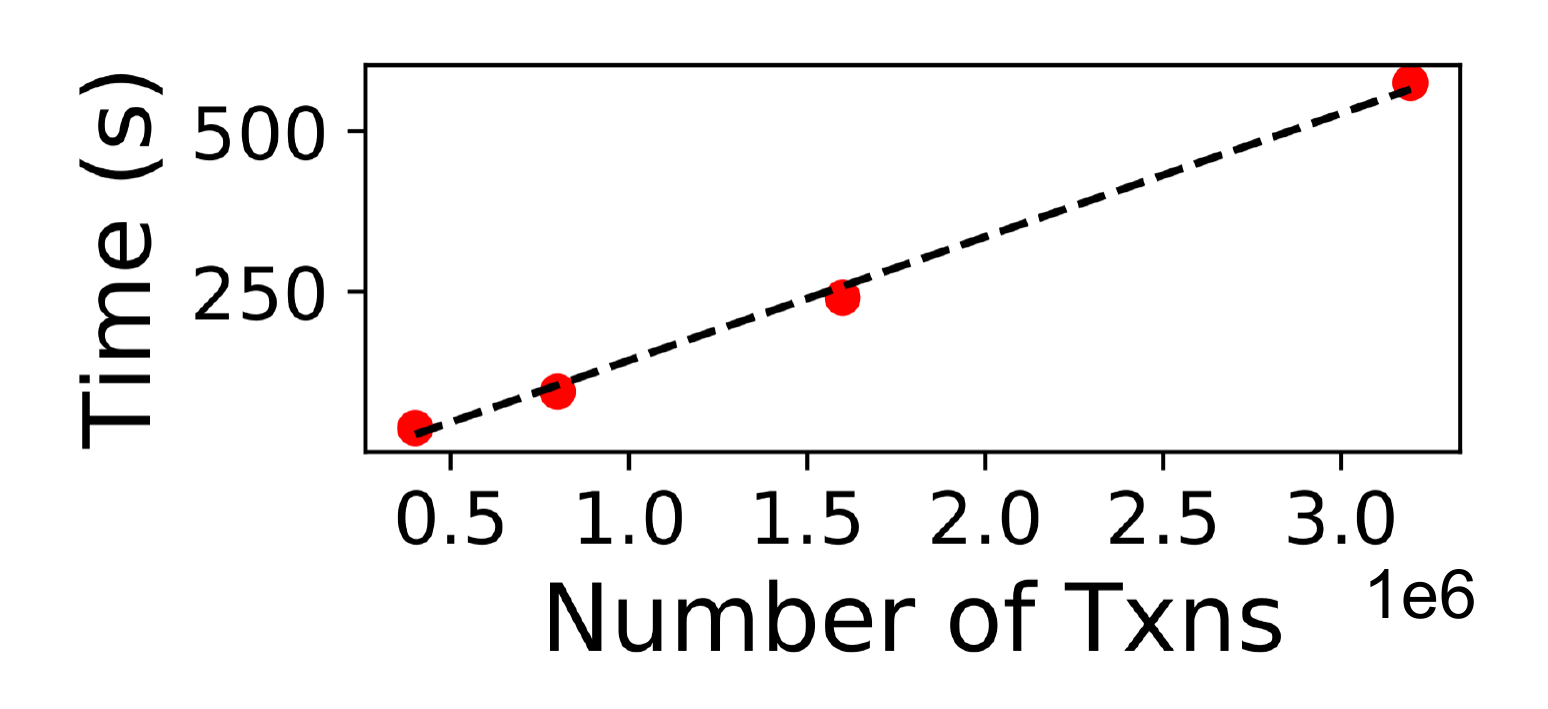}
\caption{\emphasize{\method is scalable}: 
which scales linearly on input size, 
and takes only 10 minutes for 3.2M transactions. \label{fig:scale}}
\end{figure}

\begin{figure*}[t!]
    \centering
    \begin{subfigure}[t]{0.3\textwidth}
        \centering
        \includegraphics[height=2.35in]{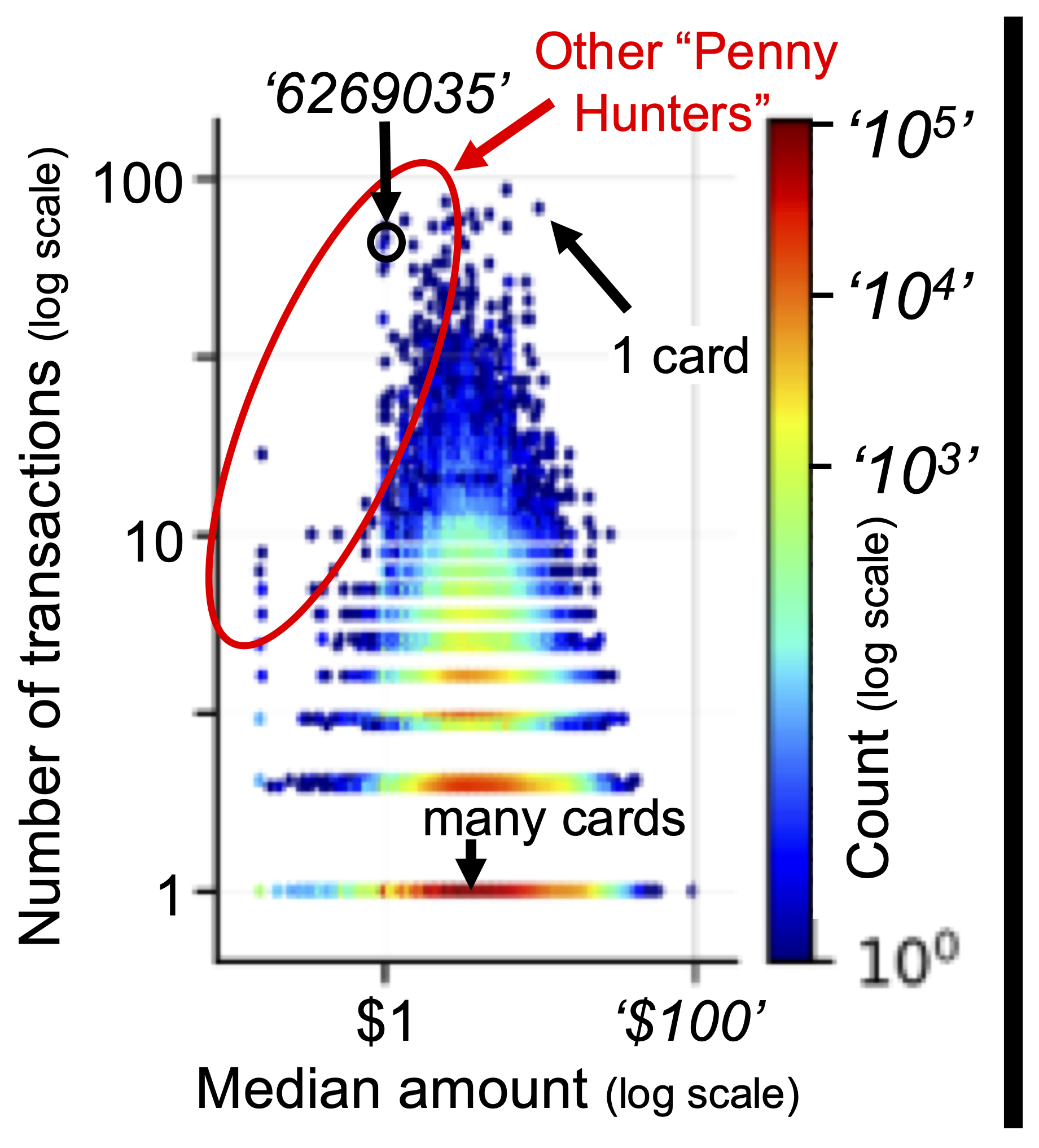}
        \caption{Goal G1 -- \gone}
        \label{fig:ph_others}
    \end{subfigure}%
    ~ 
    \begin{subfigure}[t]{0.7\textwidth}
        \centering
        \includegraphics[height=2.35in]{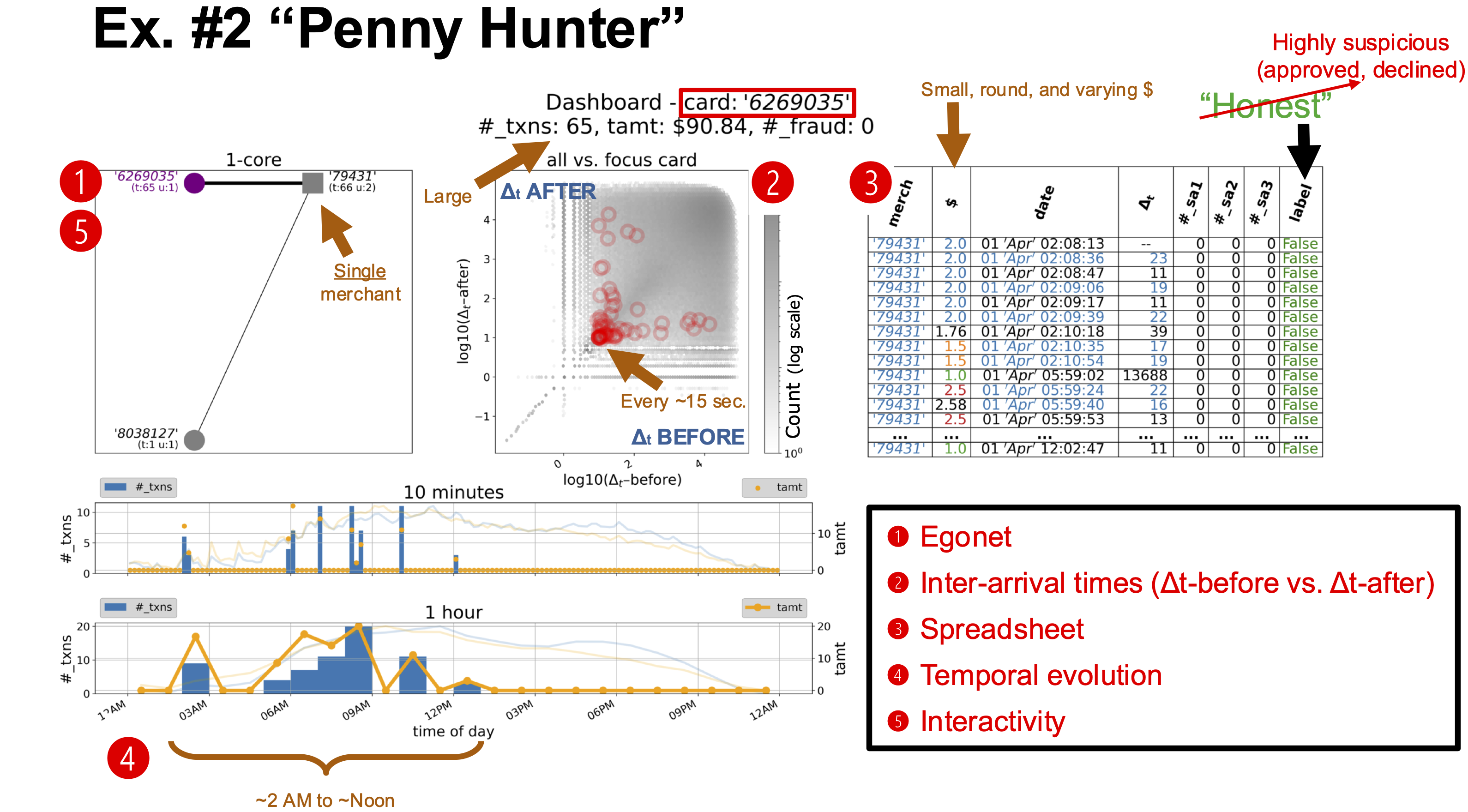}
        \caption{Goal G2 -- \gtwo}
        \label{fig:ph_example}
    \end{subfigure}
    \caption{\emphasize{Example of \ph} \label{fig:ph}: 
    Notice the high count of small-value transactions, every $\approx$15 seconds.}
\end{figure*}



\begin{figure*}[t!]
    \centering
    \begin{subfigure}[t]{0.3\textwidth}
        \centering
        \includegraphics[height=2.325in]{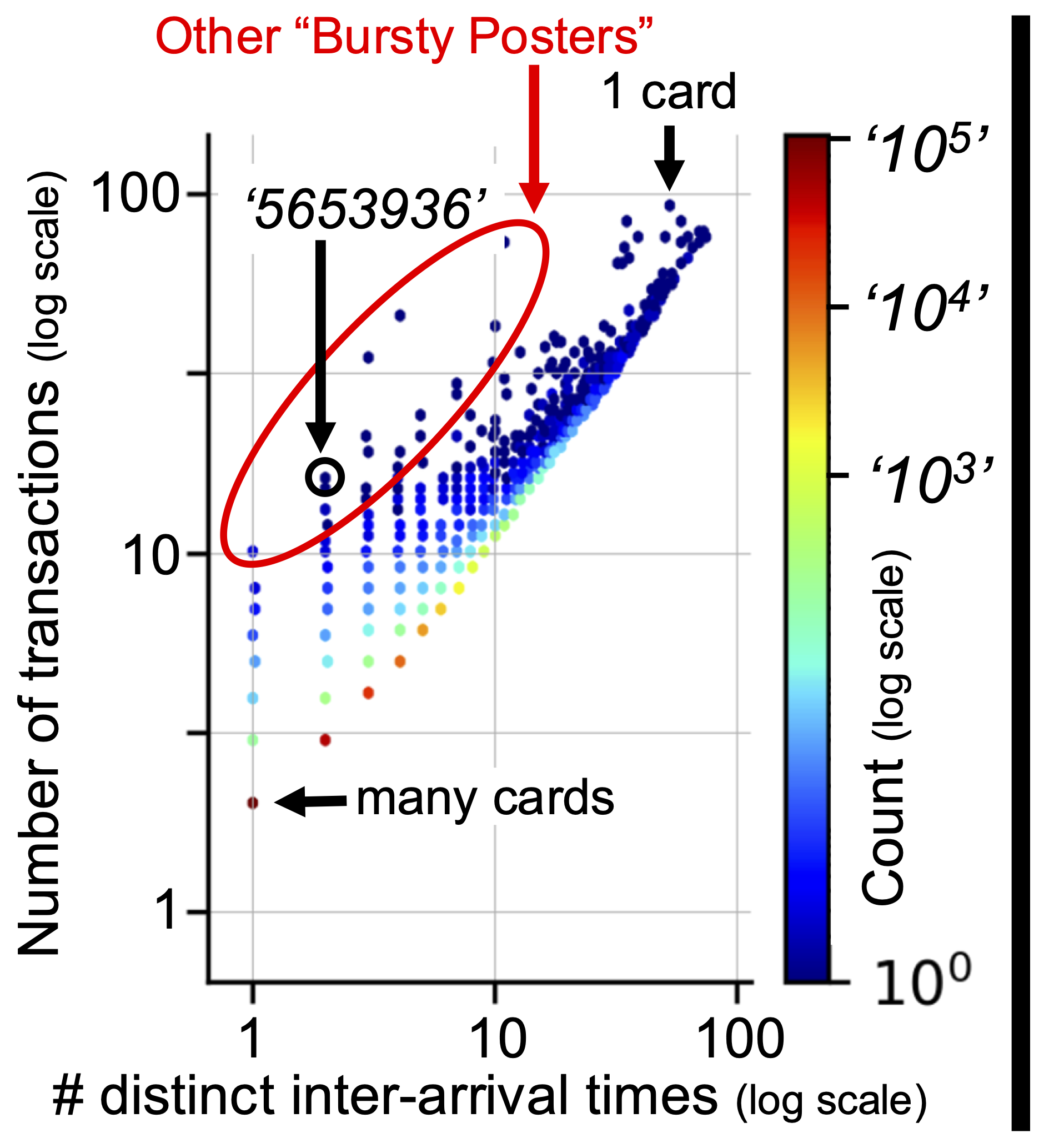}
        \caption{Goal G1 -- \gone}
        \label{fig:cp_others}
    \end{subfigure}%
    ~ 
    \begin{subfigure}[t]{0.7\textwidth}
        \centering
        \includegraphics[height=2.325in]{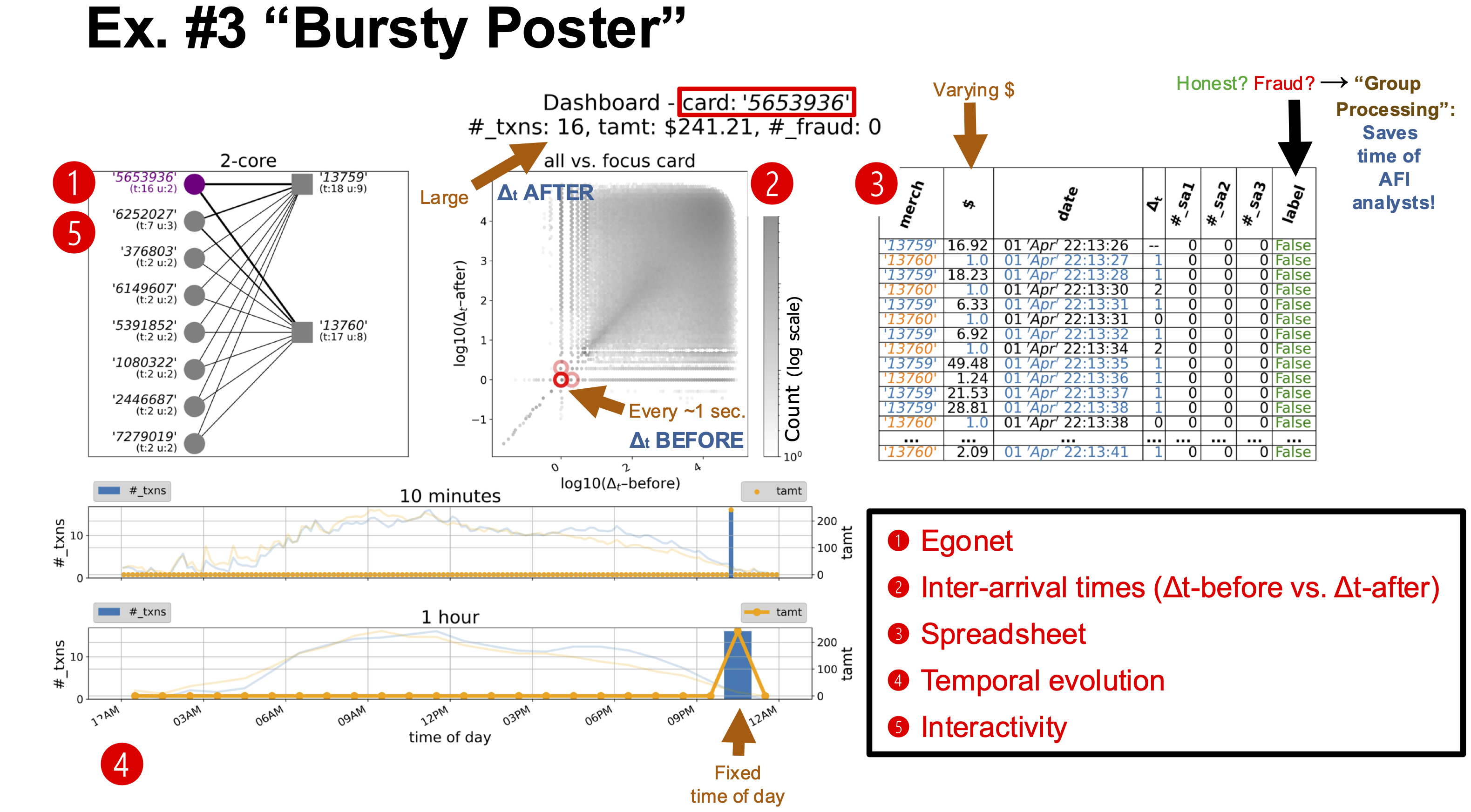}
        \caption{Goal G2 -- \gtwo}
        \label{fig:cp_example}
    \end{subfigure}
    \caption{\emphasize{Example of \cp:} \label{fig:cp}
    Notice the bursty activity, at a strange time of day (10~PM).}
\end{figure*}

\begin{lemma}
After a linear preprocessing step, 
    \methodtwo also requires  time linear on the number of txns
    and neighbors for the card under investigation.
\end{lemma}

\begin{proof}
    (sketch): Again, all the elements we need to plot for the dashboard
    require a linear pass on the related information.
\end{proof}

In fact, the dashboard usually takes a few seconds on a stock machine,
for the million-scale
dataset we are working with.

\section{Experiments - Discoveries}
\label{sec:exp}

\hide{
\begin{table*}[ht]
\begin{center}
\begin{tabular}{r|r|l}
   Nodes & Edges & Description  \\ \hline \hline
  \multicolumn{3}{l}{\textbf{Real-world Networks}} \\ \hline \hline 
  13,579 & 37,448 & AS Oregon \\  
  23,389 & 47,448 & CAIDA AS 2004 to 2008  \\
  \hline \hline
\end{tabular}
\end{center}
\caption{Summary of real-world networks used.}
\label{tab:datasets}
\end{table*}
}

%

%

Thanks to our method \method, we processed real data \afidataset\xspace
and discovered new types of suspicious behavior.
The \afi analysts are currently investigating them,
and one of them (\dmg) appears to be a good predictor of fraud.
For ease of reference, we gave these three new behaviors
the following names: \dmg, \ph, \cp, for reasons
that will be clear next.



\subsection{\dmg}
One of the most interesting cases we identified are cards with many, synchronized txns of the same amount.
We call them \dmg~because they exhibit two distinct types of machine-gun behavior, by repeating the amounts and also repeating the inter-arrival times of the txns, as shown in Figure~\ref{fig:dmg_others}.
Figure~\ref{fig:dmg_example} shows the dashboard for one example \dmg~card \dmgexample.
This particular card had $66$ txns, mostly of $\$0.99$, occurring at every $\sim3$ minutes within a day.
Thanks to \methodtwo, we also note that the majority of the txns were made with a single, small merchant, which makes the case more suspicious.
Despite being initially considered to be honest, these txns were later confirmed as fraud by the \afi~ collaborators after receiving an alarm from \methodone and reviewing the case with \methodtwo.

\subsection{\ph}
An additional relevant case that we found is cards with many synchronized txns of varying, round, and small amounts.
These cards also repeat the inter-arrival times of the txns like those of the type \dmg, but, differently, they vary the amounts with a tendency to keep them small and round, rarely surpassing $\$5$ probably to avoid being caught, as shown in Figure~\ref{fig:ph_others}.
We call them \ph~to emphasize the small-amount characteristic.
Figure~\ref{fig:ph_example} shows the dashboard for an example \ph~card \phexample.
This card in particular had $65$ txns of small, round amounts like $\$1$, $\$2$, $\$1.5$, etc., every $\sim15$ seconds within a day.
Thanks to our \methodtwo, we also verify that all the txns were made with a single merchant that had only one other txn in the day. 
It raises a suspicion that the owner of this card and the merchant may be colluding with one another.
Once again, despite being considered honest at first glance, these txns were later {\bf confirmed} to be 
highly suspicious 
by our \afi~collaborators thanks to the alarm 
raised by our \methodone and the visual support of \methodtwo.

\subsection{\cp}
Another relevant case we identified are cards that made many, synchronized txns in a fixed time of the day with rarely round and potentially large amounts.
We call them \cp~because they clearly depend on an automated system, provided that a human is not likely to synchronously post many txns within a short period, as shown in Figure~\ref{fig:cp_others}.
Figure~\ref{fig:cp_example} shows the dashboard for an example \cp~card \cpexample.
Thanks to \methodtwo, we can easily note that this card made $16$ txns of varying amounts like $\$49.48$, $\$2.09$, etc., at every $\sim1$ second within a minute, around $10$~PM.
Our dashboard also highlights that the txns were made with only two, small merchants, which may indicate collusion.
Initially considered honest, these txns were reviewed by \afi~human experts following the alert of \methodone.
The experts confirmed that the case is highly unusual, but could not tell if it is fraud or honest.
In either way, the discovery of the \cp~behavior is valuable to \afi~because it supports the development of group-decision systems for edge cases such as this one.

\section{Conclusions}
\label{sec:concl}


We presented \method, a system
that addresses our two design goals (\gone and \gtwo):
\method
helps financial analysts spot {\em new types} of fraud,
and justify their decisions.
Our method \method meets our design goals:
\begin{compactitem}
    \item {\emphdetection}: \methodone uses carefully designed features,
    and provides heatmaps where analysts can spot 
    suspicious micro-clusters.
    \item {\emphjustification}: The heatmaps and 
    our dashboard helps analysts easily discover more evidence,
    before they take remedial action.
    \item {\em \effective}: \method has already discovered
    three new types of behavior, one of which 
    (`double-machine-gun') is very suspicious.

\end{compactitem}

Moreover, \method scales linearly with the input size,
     as shown in Figure~\ref{fig:scale}.
     Thanks to its desirable properties above,
     \method is considered for \textbf{productization} within the \afi. 
The idea is to capitalize on \method to detect \emph{new types} of fraudulent behavior (unknown unknowns), before the domain experts are even aware of it, while still relying on classification to find new instances of fraud types that are already known.


{\bf Reproducibility:} 
Pending approval by the \afi, we plan to open-source our code
and provide synthetic data as sample input.


\bibliographystyle{abbrv}
\bibliography{BIB/ref}

\clearpage
\appendix


\section{Additional Experiments}
\label{app:features}
As mentioned earlier, the \afi has been using $F \approx 400$
numerical features for each transaction;
and it also has some fraud/non-fraud labels for a few transactions.

Our goal is to find a small subset of these features,
that would (a) work well on the existing, known types of fraud,
and (b) hopefully, also help in the new, unknown types
of fraud that \method is aiming for.

Thus, we did feature selection, and compare some
state-of-the-art (SOTA) supervised methods  with all the $F$
features,
against an unsupervised method (Isolation Forest)
with a few features.
Notice that this is an unfair comparison, favoring the supervised methods.

Yet, as shown in Table~\ref{tab:add}, 
the unsupervised method (Isolation Forest with 5 features) 
performs well, coming in second place consistently, and often very close
to the winner.

\begin{table}[htbp]
\caption{\emphasize{A few features are enough:} 
Unsupervised method with 5 features competes well
against SOTA classifiers (\colorbox{green}{first} and 
\colorbox{green!25}{second} place).
\label{tab:add}}
\setlength\fboxsep{0pt}
\centering{\resizebox{1\columnwidth}{!}{
\begin{tabular}{c | c | ccc}
\toprule
\textbf{Method} & \textbf{\# of Feat.} & \textbf{AP} & \textbf{Prec@100} & \textbf{Prec@1000} \\
\midrule
Random & N/A & 00.2$\pm$0.0 & 00.1$\pm$0.4 & 00.2$\pm$0.1 \\
\midrule
XGBoost & > 400 & \gold{24.2$\pm$0.0} & 71.0$\pm$0.0 & \gold{71.2$\pm$0.0} \\
Logistic Reg. & > 400 & 06.3$\pm$0.0 & 32.0$\pm$0.0 & 16.9$\pm$0.0 \\
Random Forest & > 400 & 08.5$\pm$0.5 & \gold{86.6$\pm$1.6} & 41.5$\pm$2.1 \\
\midrule
Isolation Forest & 5 
                & \silver{21.4$\pm$0.0} 
                & \silver{85.6$\pm$0.5} 
                & \silver{47.9$\pm$0.1} \\
\bottomrule
\end{tabular}
}}
\end{table}

This is an indication that the 5 chosen features are inherently useful,
and they may also help in spotting new types of fraud.


\end{document}